\documentclass{article}
 \PassOptionsToPackage{numbers, compress}{natbib}

  \usepackage[preprint]{neurips_2020}

\usepackage[utf8]{inputenc} 
\usepackage[T1]{fontenc}    
\usepackage{hyperref}       
\usepackage{url}            
\usepackage{booktabs}       
\usepackage{amsfonts}       
\usepackage{nicefrac}       
\usepackage{microtype}      
\usepackage{microtype}
\usepackage{graphicx}
\usepackage{subcaption}
\usepackage{booktabs} 

\usepackage{multicol} 

\usepackage{hyperref}

\usepackage{tikz}
\usetikzlibrary{positioning,fit,backgrounds,shapes.geometric, decorations.text}

\usepackage[utf8]{inputenc} 
\usepackage[T1]{fontenc}    
\usepackage{hyperref}       
\usepackage{url}            
\usepackage{booktabs}       
\usepackage{amsfonts}       
\usepackage{nicefrac}       
\usepackage{microtype}      

\usepackage{graphicx}
\usepackage{placeins}
\usepackage{natbib}
\usepackage{algorithm}
\usepackage{algorithmic}
\usepackage{amssymb}
\usepackage{amsmath}

\usepackage{amsthm} 

\newcommand{\beq}[1][\vspace{0.3em}]{#1\begin{equation}}
\newcommand{\eeq}{\end{equation}}

\newcommand{\bit}{\vspace{0mm}\begin{itemize}}
\newcommand{\eit}{\vspace{0mm}\end{itemize}}
\newcommand{\ben}{\vspace{0mm}\begin{enumerate}}
\newcommand{\een}{\vspace{0mm}\end{enumerate}}
\newtheorem{theorem}{Theorem}

\newtheorem{lemma}[theorem]{Lemma}
\newcommand{\bb}[1]{\mathbb{#1}}

\newtheorem{example}{Example} 

\usepackage{color}

\DeclareMathOperator*{\argmax}{arg\,max}
\DeclareMathOperator*{\argmin}{arg\,min}

\newcommand\drawnestedsets[4]{
  \def\position{#1}
  \def\nbsets{#2}
  \def\listofnestedsets{#3}
  \def\reversedlistofcolors{#4}

  \coordinate (circle-0) at (#1);
  \coordinate (set-0) at (#1);
  \foreach \set [count=\c] in \listofnestedsets {
    \pgfmathtruncatemacro{\cminusone}{\c - 1}
    \node[right=3pt of circle-\cminusone,inner sep=0]
    (set-\c) {$\set$};
    \node[ellipse,inner sep=0,fit=(circle-\cminusone)(set-\c)]
    (circle-\c) {};
  }

  \begin{scope}[on background layer]
    \foreach \col[count=\c] in \reversedlistofcolors {
      \pgfmathtruncatemacro{\invc}{\nbsets-\c}
      \pgfmathtruncatemacro{\invcplusone}{\invc+1}
      \node[ellipse,draw,fill=\col,inner sep=0,
      fit=(circle-\invc)(set-\invcplusone)] {};
    }
  \end{scope}
}

\usepackage{listings} 

\title{Follow the Neurally-Perturbed Leader for Adversarial Training}

\author{%
  Ari ~Azarafrooz }

\begin{document}

\maketitle

\begin{abstract}
Game-theoretic models of learning are a powerful set of models that optimize multi-objective architectures. Among these models are zero-sum architectures that have inspired adversarial learning frameworks. An important shortcoming of these zeros-sum architectures is that gradient-based training leads to weak convergence and \textit{cyclic} dynamics. 

We propose a novel `follow the leader' training algorithm for zeros-sum architectures that guarantees convergence to mixed Nash equilibrium \textit{without} cyclic behaviors. It is a special type of `follow the \textit{perturbed} leader' algorithm where perturbations are the result of a neural \textit{mediating agent}. 

We validate our theoretical results by applying this training algorithm to games with convex and non-convex loss as well as generative adversarial architectures. Moreover, we customize the implementation of this algorithm for adversarial imitation learning applications. At every step of the training, the mediator agent perturbs the observations with generated \textit{codes}.  As a result of these mediating codes, the proposed algorithm is also efficient for learning in environments with \textit{various factors of variations}. We validate our assertion by using a procedurally generated game environment as well as synthetic data. Github implementation is available \href{https://github.com/azarafrooz/FTNPL}{here}.
 \end{abstract}

\section{Introduction}

A wide range of recent learning architectures expresses the learning formulation as a multi-objective and game-theoretic problem. They are useful to build log-likelihood free deep generative models \cite{Goodfellow2014GenerativeAN, Schmidhuber1992LearningFC}, learn disentangled representations \cite{Chen2016InfoGANIR, Li2017InfoGAILII}, learn adversarial imitation \cite{Ho2016GenerativeAI}, learn complex behaviors \cite{Wang2019PairedOT}, incorporate hierarchical modeling to mitigate the reinforcement (RL) exploration issues \cite{Vezhnevets2017FeUdalNF, Kulkarni2016HierarchicalDR, Azarafrooz2019HierarchicalSA}, formulate curiosity \cite{Schmidhuber1991APF,Schmidhuber2019UnsupervisedMA, Pathak2017CuriosityDrivenEB} and imagination objectives in RL \cite{Racanire2017ImaginationAugmentedAF}, tighten the lower bound for mutual information estimates \cite{Poole2019OnVB}, compute synthetic gradients \cite{Jaderberg2016DecoupledNI}, etc. However, the behavior of the gradient-based methods of training in these architectures is even more complicated than those of single objective ones. For example,  \cite{Balduzzi2018TheMO, Bailey2019MultiAgentLI} show that gradient-based methods suffer from recurrent dynamics, slow convergence and inability to measure convergence in zero-sum type games. The existence of cyclic behavior necessitates a slow learning rate and convergence. \cite{Balduzzi2018TheMO} proposed a new gradient-based method by utilizing the dynamics of Hamiltonian and Potential games.

The focus of our paper is on training \textit{adversarial} architectures using \textit{regret minimization} framework \cite{Grnarova2017AnOL,Hazan2017EfficientRM,Kodali2018OnCA}. Adversarial training using regret minimization framework also suffers from \textit{cyclic} behaviors \cite{Mertikopoulos2018CyclesIA, Bailey2019MultiAgentLI, Bailey2018MultiplicativeWU, Piliouras2014OptimizationDC}. However, as we show in this paper, it provides a mathematically elegant framework for designing novel training algorithms that avoid cyclic dynamics. Another difficulty of standard regret minimization methods is that they fail in \textit{non-convex} settings. To address the non-convexity issue, \cite{Gonen2018LearningIN} invokes an offline oracle to introduce random noise. We propose a novel algorithm to address both of these issues. We show that a neural network mediator can remove the cyclic behaviors by perturbing the dynamics of the game. Moreover, there is no need for convexity assumption in our approach. 
The mediator perturbs the dynamic of the game by augmenting the observation of the players with so-called `correlated' codes. The nature of such codes is related to the notion of correlated equilibrium \cite{Aumann1987CorrelatedEA, Blum2005FromET} and similar to disentangled codes in \cite{Chen2016InfoGANIR, Li2017InfoGAILII}. Thanks to these correlated codes, the proposed method is also an efficient approach for learning in multi-modal and adaptive environments. 
\section{Problem Formulation}
\subsection{Game Theoretical Preliminaries}\label{game_definitions}
\begin{figure}[ht]
\begin{center}
\begin{tikzpicture}\tiny
\drawnestedsets{0,0}{4}{PNE,MNE,\textbf{\color{white}{CorEq}}, CCE}{white,gray,white, white}
\end{tikzpicture}
\caption{\small Generalizations of pure Nash equilibria. `PNE' stands for pure Nash equilibria; `MNE'
for mixed Nash equilibria; `CorEq' for correlated equilibria; and `No Regret (CCE)' for coarse
correlated equilibria.}
\label{equilibria:fig}
\end{center}
\end{figure}
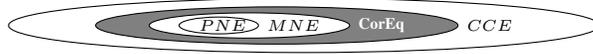
 Consider a zero-sum game between two players, agent $\pi$ and discriminator $D$ with strategies $\phi \in \Phi, \omega \in \Omega$ respectively. In a deep learning game, each player is a neural network with its parameters as strategies. Let $L(\phi,\omega)$ denote the loss of the game and $\mu \in \triangle  ~ \Phi \times \Omega $ be \textit{joint} mixed strategy, where $\bigtriangleup$ and $\times$ denote the probability simplex and cartesian product respectively. Let us also define the \textit{marginal} mixed strategies $\mu_{\pi}, \mu_{D}$. For example, $\mu_{\pi}(\phi)$ is the probability that an agent plays strategy $\phi$. A Nash equilibrium gets achieved when no player has an incentive to deviate unilaterally. If a Nash equilibrium gets achieved over the probability distributions of strategies, it is a mixed Nash equilibrium (MNE). If one relaxes the assumptions on the equilibrium, other game-theoretic concepts can be derived, as shown in Fig. \ref{equilibria:fig}. The least restrictive of these is known as \textit{coarse} correlated equilibria (CCE). It corresponds to the \textit{empirical} distributions that arise from the repeated joint-play by no-regret learners.  An essential concept required for the development of our algorithm is correlated equilibrium (CorEq) \cite{Aumann1987CorrelatedEA}. Computing CorEq amounts to solving a linear program. As a result, it is computationally less expensive than computing NE, which amounts to computing a fixed point.  However, CorEq is concerned with \textit{joint} mixed strategies (e.g $\mu(\phi,\omega))$ which is and stronger and more restrictive notion than CCE. We therefore first define a simpler notion of MCorEq. We are particularly interested in its maximum entropic version similar to \cite{Ortiz2006MaximumEC}.
 
 \textbf{Marginal correlated equilibrium (MCorEq)}:  Let $u_{\pi}(\phi) = \bb{E}_{\omega \sim \mu_{D}}L(\phi,\omega)$, $u_{D}(\omega) = - \bb{E}_{\phi \sim \mu_{\pi}}L(\phi,\omega)$ be the \textit{marginal loss}. Also let $g_D(k,\hat{k}) =u_D[\hat{k}]-u_D[k]$ and $g_\pi(k,\hat{k}) = u_{\pi}[\hat{k}] - u_{\pi}[k]$ be discriminator's and agent's \textit{marginal payoff gain} for selecting strategy $k$ instead of $\hat{k}$. Marginal payoff $g$ quantifies the motivation of the users to switch to other strategies. MCorEq is then a marginal mixed strategy $(\mu_{\pi}, \mu_{D})$ such that $\forall (\phi, \hat{\phi}), \forall (\omega, \hat{\omega}) $, $\mu_{\pi}(\phi)>0, \mu_{D}(\omega)>0$, $g_\pi(\phi, \hat{\phi}) <=0$ and $g_D(\omega,\hat{\omega}) <=0$, i.e no user benefits (in the marginal loss sense) from switching strategies.

\textbf{MeMCorEq}: It is the solution to the convex optimization problem of $\mu^{*} = \argmax_{\mu \in MCorEq} H(\mu)$
 where $H$ is the entropy. It is therefore a convex optimization problem. 
\subsection{No-regret Learning}
One gradient-based method of learning in games is no-external-regret learning. External regret is the difference between the actual loss achieved and the smallest possible loss that could have been achieved on the sequence of plays by playing a \textit{fixed} action. For example in the context of the aforementioned zero-sum game, the regret for $\pi$ and $D$ is $ \displaystyle  \max_{\phi \in \Phi}  \sum_{t=0}^{T-1}  L(\phi_t,\omega_t) -  L(\phi,\omega_t)$ and $ \displaystyle \min_{\omega \in \Omega} \sum_{t=0}^{T-1}  L(\phi_t,\omega_t) -  L(\phi_t,\omega)$ respectively.  Regret minimization algorithms ensure that long term regret is \textit{sublinear} in the number of time steps.  It is known that the optimal minimax regret of zeros-sum games is $\mathcal{O}(\log(T))$ \cite{Freund1999AdaptiveGP, Rakhlin2013OptimizationLA}. There are several classes of algorithms that can yield sub-linear regret. One well-known class of no-regret learners is \textit{Follow The Regularized Leader (FTRL)}\cite{ShalevShwartz2006ConvexRG}:
 \begin{equation}
\begin{split}
\label{ftrl:eq}
\phi_t = \argmin \sum_{i<t} L(\phi_i, \omega_i) + h(\phi_i) ~~   ;  ~~
\omega_t = \argmax \sum_{i<t} L(\phi_i,\omega_i) - h(\omega_i) 
\end{split}
\end{equation} 
 
 One common choice is $\ell_2$ regularization $h(\theta) =  \| \theta \|^2$ which we used in the following examples to illustrate the problem visually. FTRL algorithms are not suitable for non-convex losses. To address the non-convex optimization case,  \cite{Gonen2018LearningIN} proposed a \textit{follow the perturbed leader (FTPL)} by choosing $h(\theta)=\sigma \theta$ where $\sigma \sim (Exp(\zeta))$ is an exponential noise introduced by an oracle.
 
 \textbf{Weak convergence and cyclic dynamics}
 Let $h^T=(\mu_0,...,\mu_T)$ be the history of past strategies when the game is played repeatedly up to time $T$ and $\mu_t=(\phi_t,\omega_t)$. 
 It is known that that no-regret dynamics \textit{weakly converge} to MNE in zero-sum games \cite{Freund1999AdaptiveGP, Immorlica2011DuelingA}. Weakly convergence implies that time average of $h^T$ will converge to the MNE. Previous works relied on this weaker notion of convergence for training adversarial networks \cite{Grnarova2017AnOL}. However, this weaker notion can be misleading. It is as meaningful as the statement that ``moon converges to earth'' instead of stating that the moon follows a trajectory that has the earth as its center \cite{Papadimitriou2016FromNE}. Weak convergence is also related to the dynamics of FTRL in adversarial games which is known to exhibit recurrent dynamics \cite{Mertikopoulos2018CyclesIA}. This cyclic behavior is common across all choices of regularizers $h$ and learning rates. We also show in the following example that weak convergence and cyclic behavior also hold true for the FTPL. 
\begin{example}\label{example1}
Consider a specific type of zero-sum game known as matching Pennies game with $L(\phi,\omega)=\phi\textbf{A}\omega^T$ with $A=\begin{bmatrix}
    1 & -1  \\
    -1 & 1
  \end{bmatrix}$. MNE is $\mu^*=(0,0)$ for this game. 
\end{example}

Trajectories of FTRL training dynamic for this example is visualized in Fig. \ref{fig:cycle_FTRL} and Fig. \ref{fig:cycle_FTRL_large_lr}. 
Note in Fig. \ref{fig:cycle_FTRL_large_lr}, how a large learning rate leads to an even weaker convergence (in the sense defined above). Fig. \ref{fig:FTPL_convex} shows that FTPL also exhibits recurrent dynamics. Aside from the weak convergence, another implication of this cyclic behavior is the slow learning rate, as discussed in \cite{Balduzzi2018TheMO}. This is because gradient-based algorithms do not follow the steepest path toward fixed points due to the `rotational force'. 

\textbf{Non-convexity} FTRL with $\ell_2$ regularization does not converge in non-convex situations as shown in the following example.
\begin{example}\label{example2}
Assume $\pi$ to have the same loss as in Ex. \ref{example1} but let the loss for $D$ to be defined $\text{ReLU}(-\phi\textbf{A}\omega^T)$ where $\text{ReLU}(x) = \max(0,x)$. MNE stays the same $\mu^*=(0,0)$ but FTRL does not converge to MNE as shown in Fig \ref{fig:FTRL_nonconvex}. Unlike FTRL, FTPL weakly converges to MNE. We visualized the learning trajectories of FTPL in \ref{fig:FTPL_nonconvex}.
\end{example}

We present an approach that does not require convexity assumption and that converges (instead of weak convergence) to MNE without cyclic behavior.

 \begin{figure}
\centering
\begin{subfigure}{0.23\textwidth}
\includegraphics[width=\textwidth]{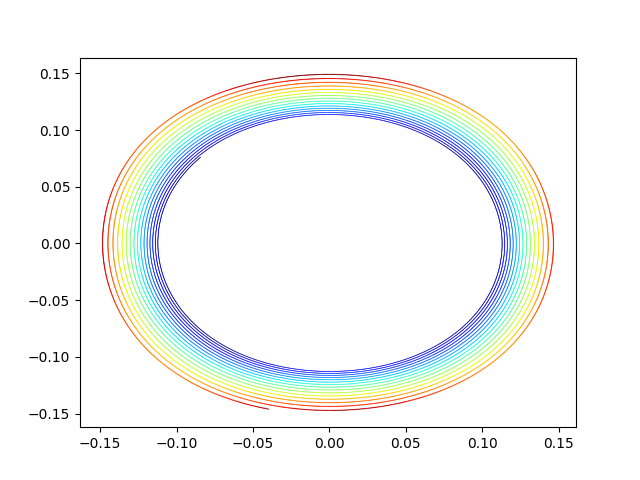}
\caption{\tiny FTRL: small leaning rate}
\label{fig:cycle_FTRL}
\end{subfigure}
~
\begin{subfigure}{0.23\textwidth}
\includegraphics[width=\textwidth]{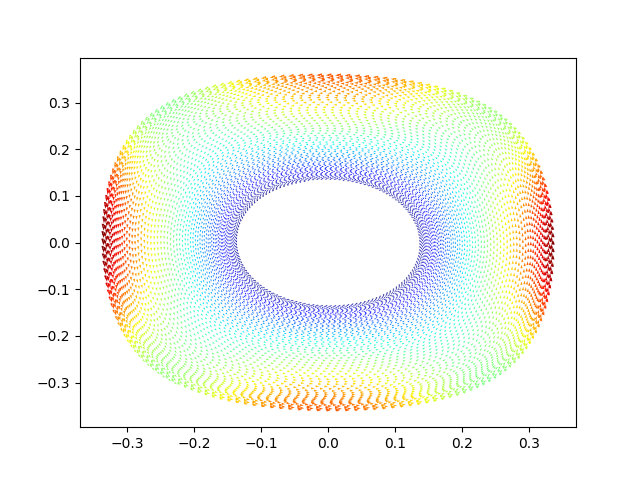}
\caption{\tiny FTRL: large learning rate}
\label{fig:cycle_FTRL_large_lr}
\end{subfigure}
~
\begin{subfigure}{0.23\textwidth}
\includegraphics[width=\textwidth]{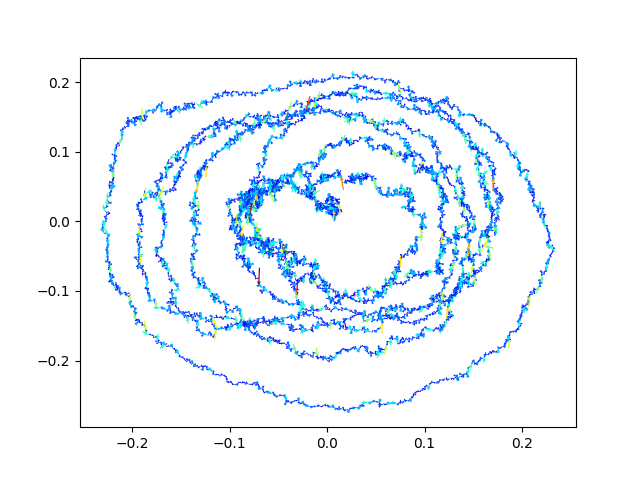}
\caption{\tiny FTPL}
\label{fig:FTPL_convex}
\end{subfigure}
~
\begin{subfigure}{0.23\textwidth}
\includegraphics[width=\textwidth]{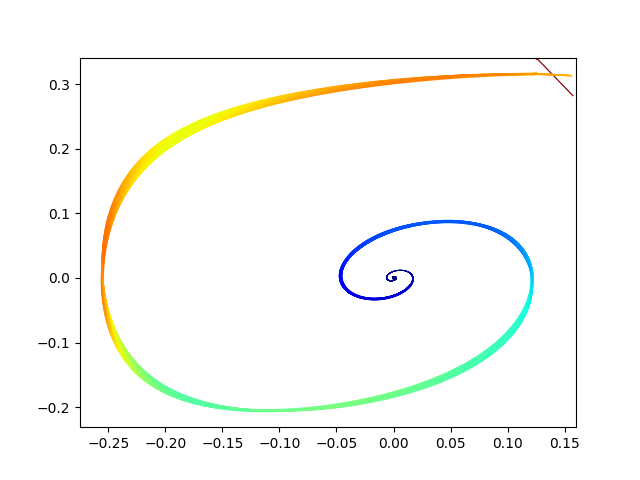}
\caption{\tiny FTNPL}
\label{fig:FTNPL}
\end{subfigure}
\caption{\small Training dynamic trajectories in example \ref{example1}. x-axis and y-axis are $\phi[0]$ and $\omega[0]$ respectively. $\|\mu_t-\mu_{t-1}\|^2$ are encoded using colors to track convergence.  Small values are blue and largest are red.}
\end{figure}

\begin{figure}
\centering
\begin{subfigure}{0.23\textwidth}
\includegraphics[width=\textwidth]{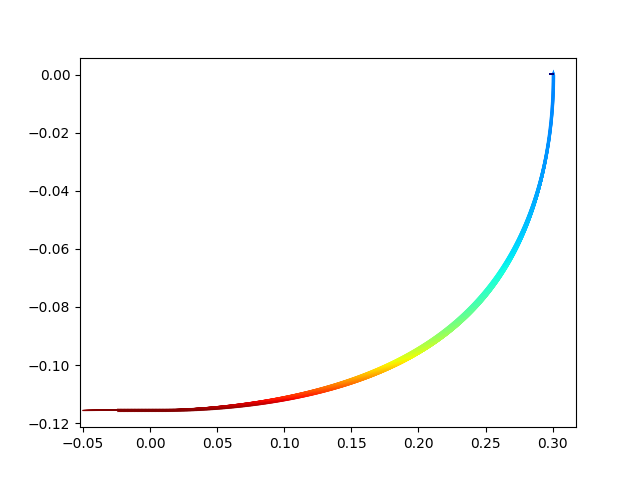}
\caption{\tiny FTRL}
\label{fig:FTRL_nonconvex}
\end{subfigure}
~
\begin{subfigure}{0.23\textwidth}
\includegraphics[width=\textwidth]{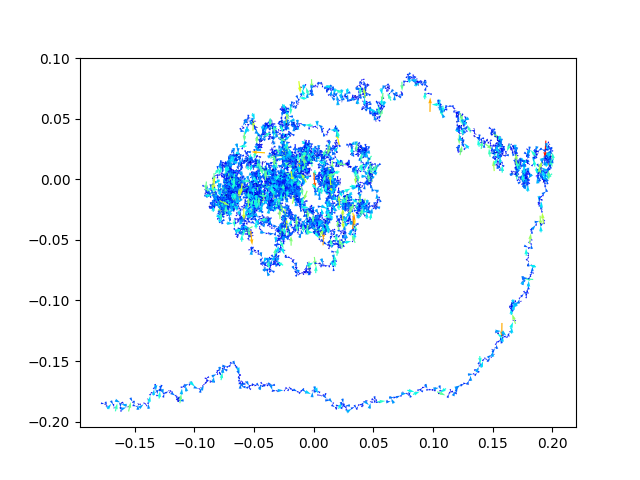}
\caption{\tiny FTPL}
\label{fig:FTPL_nonconvex}
\end{subfigure}
~
\begin{subfigure}{0.23\textwidth}
\includegraphics[width=\textwidth]{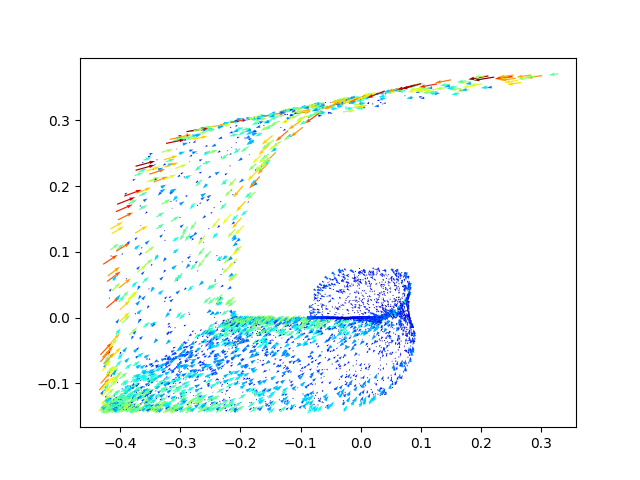}
\caption{\tiny FTNPL with one code}
\label{fig:FTNPL_nonconvex_1code}
\end{subfigure}
~
\begin{subfigure}{0.23\textwidth}
\includegraphics[width=\textwidth]{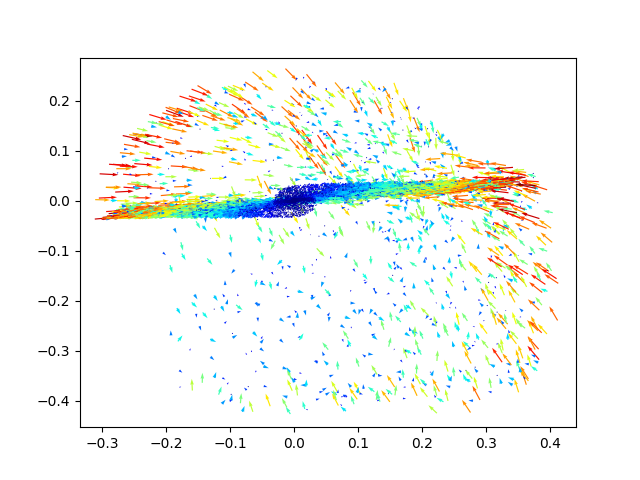}
\caption{\tiny FTNPL with two codes}
\label{fig:FTNPL_nonconvex}
\end{subfigure}
~

\caption{\small Training dynamic trajectories in example \ref{example2}. x-axis and y-axis are $\phi[0]$ and $\omega[0]$ respectively. $\|\mu_t-\mu_{t-1}\|^2$ are encoded using colors to track convergence.  Small values are blue and largest are red.}
\end{figure}
 \section{Follow the Neurally-Perturbed Leader} \label{Fig:FTNPL_scheme}
Another choice for regularizer $h(\theta)$ in FTRL is the entropy function $H(\theta)$. This choice of regularizer leads to Multiplicative Weights (MW) algorithm \cite{Arora2012TheMW}. 

 \begin{theorem}\label{MW_zeo_sum_dynamic:thm}
 When a continuous time MW algorithm is applied to zero-sum game with a fully mixed Nash equilibrium $\mu^* = (\mu^*_D, \mu^*_{\pi})$, cross entropy between each evolving strategy $\mu(t)$ and the players' mixed Nash equilibrium $H(\mu^*, \mu) = - \sum_{\phi \in \Phi} \mu^*_{\pi}(\phi)  \ln \mu_{\pi}(\phi)  -  \sum_{\omega \in \Omega} \mu^*_D(\omega) \ln \mu_D(\omega)$ remains constant. 
 \end{theorem} 
 
\begin{proof}
\small 
 Various works including \cite{Hofbauer2009TimeAR, Mertikopoulos2014LearningIG} have shown that the continuous-time version of MW algorithm follows a replicator dynamic \cite{Taylor1978EvolutionarilySS} as is described by:
 \begin{equation}
 \small
\begin{split}
\label{replicator:eq}
\frac{\dot{\mu}_{\pi}(\phi)} {\mu_{\pi}(\phi)}= u_{\pi}(\phi)-\sum_{\hat{\phi} \in \Phi}  \mu_{\pi}(\hat{\phi}) u_{\pi}(\hat{\phi}) ~~   ;  ~~
\frac{\dot{\mu}_{D}(\omega)}{ \mu_{D}(\omega)} = u_{D}(\omega)-\sum_{\hat{\omega} \in \Omega}  \mu_{D}(\hat{\omega}) u_{D}(\hat{\omega})
\end{split}
\end{equation}
where as before $u_{\pi}(\phi), u_{D}(\omega)$ are the marginal loss and $\dot{u}=\frac{du}{dt}$.

It suffices to take time derivative of cross entropy term and plug in Eq. \ref{replicator:eq} and note the zero-sum nature of the game:
 \begin{equation}
 \small
\begin{split}
\label{corss_entropy:eq}
\frac{dH(\mu^*, \mu(t))}{dt} = - \sum_{\phi} \mu^*_{\pi}(\phi) \frac{\dot{\mu}_{\pi}(\phi)}{\mu_{\pi}(\phi)} - \sum_{\omega} \mu^*_D(\omega) \frac{\dot{\mu}_D(\omega)}{\mu_D(\omega)} = \\
 - \sum_{\phi} \mu^*_{\pi}(\phi) [u_{\pi}(\phi)-\sum_{\hat{\phi} \in \Phi}  \mu_{\pi}(\hat{\phi}) u_{\pi}(\hat{\phi})] - \sum_{\omega} \mu^*_D(\omega) [u_{D}(\omega)-\sum_{\hat{\omega} \in \Omega}  \mu_{D}(\hat{\omega}) u_{D}(\hat{\omega})]=0
\end{split}
\end{equation}
 \end{proof}

\begin{figure}[b]
\tikzset{partial ellipse/.style args =
  {#1:#2:#3}{insert path={+ (#1:#3) arc (#1:#2:#3)}}}
\begin{tikzpicture}[>=latex]
  \draw [fill=white!90]    (3,-1.8) ellipse    (4cm and 1 cm);
  \draw [fill=white!90] (3,-1.8) ellipse (3cm and 0.75 cm);
  \draw [fill=white!90]  (3,-1.8) ellipse  (2cm and 0.5 cm);

  \shade [ball color=white] (3,-1.8) circle (.4);
  \node (equilibrium) at (3,-1.8) {$\mu^*$};
  \draw (3,-1.8) [partial ellipse=220:320:2cm and 0.5cm]
        (3,-1.8) [partial ellipse=220:320:3cm and 0.75cm];

  \shade [ball color=white] (1.3,-1.8) circle (.4);
  \node (mut3) at (1.3,-1.8) {$\mu_{t+3}$}; 

  \shade [ball color=white] (4,-1.1) circle (.4);
  \node (mut2) at (4,-1.1) {$\mu_{t+2}$}; 

  \shade [ball color=white] (5.75,-2.5) circle (.4);
  \node (mut1) at (5.75,-2.5) {$\mu_{t+1}$};
     
  \draw (3,-1.8) [partial ellipse=45:120:9cm and 1.5cm];
  \shade [ball color=white] (6,-0.3) circle (.4);
  \node (mut) at (6,-0.3) {$\mu_t$};   
  \draw [line width=0.1pt,black,->,>=latex] (mut) to (mut1);  
  \draw [line width=0.1pt,black,->,>=latex] (mut1) to (mut2);    
  \draw [line width=0.1pt,black,->,>=latex] (mut2) to (mut3);    
  
  \node at (1,-0.17) {$D_{KL}(\mu^* || \mu_t)=k$};
  \node at (6,-0.9) {\small $\mathcal{M}(.)=c_{t+1}$};
  \node at (4.8,-1.8) {\small $\mathcal{M}(.)=c_{t+2}$};
  \node at (2.2,-1.3) {\small $\mathcal{M}(.)=c_{t+3}$};
\end{tikzpicture}
\caption{\small FTNPL scheme: $\forall \mu_t$ on the same orbit, $D_{KL}(\mu^* || \mu_t)=k$. Mediator $\mathcal{M}$ minimizes $D_{KL}(\mu^* || \mu_t)$ by learning from reward term defined in eq. \ref{reward:eq}.}
\label{fig:scheme}
\end{figure}
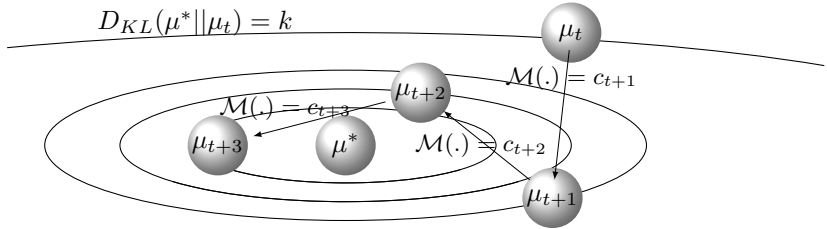

\begin{lemma}\label{convergence:lemma}
Maximizing Nash entropy $H(\mu^*)$ implies convergence of MW to MNE is no longer weak. Therefore no recurrent/cyclic dynamics exist.
\end{lemma}

\begin{proof}
\small
The degree of weakness in convergence measures of FTRL can be quantified using KL divergence $D_{KL}(\mu^* || \mu)$. Since $H(\mu^*, \mu)=H(\mu^*)+D_{KL}(\mu^* || \mu)$, and $H(\mu^*,\mu)$ is constant from Theorem \ref{MW_zeo_sum_dynamic:thm}, $D_{KL}(\mu^* || \mu) \rightarrow 0$ is equivalent to Maximizing Nash entropy $H(\mu^*)$.
\end{proof}

Let us refer to trajectory of $\forall \mu_t$ that $D_{KL}(\mu^* || \mu_t)=k$ as \textit{KL orbit} as visualized in Fig \ref{fig:scheme}. Lemma \ref{convergence:lemma} then builds a useful intuition on how to avoid the cyclic behavior and to guarantee convergence (instead of weak convergence) to the MNE. By slowly maximizing $H(\mu^*)$, we can travel toward MNE one KL orbit at a time, until we reach an orbit with radius 0, i.e $D_{KL}(\mu^* || \mu_t)=0$. At this point, convergence is no longer weak and no recurrent dynamic exists. However, it is not feasible to control $H(\mu^*)$ without interfering with the game as Nash $\mu^*$ is predetermined by $L(\phi, \omega)$. We propose to use a \textit{neural network mediator agent} $\mathcal{M}$ that perturbs the original dynamic of the game by introducing auxiliary codes to players. We will show that a proper reward for $\mathcal{M}$ can be set to maximize Nash entropy $H(\mu^*)$ under the new perturbed game. However, without a game-theoretic formulation, players will \textit{ignore} these perturbations. To address this, we use our predefined notion of MeMCorEq. The derivation MeMCorEq in our setup is not straightforward, as mediator $\mathcal{M}$ is perturbing the game dynamics sequentially as demonstrated in Fig. \ref{fig:scheme}. 

\begin{theorem} \label{thm:main} Any zero-sum game will \textbf{converge to MNE without recurrent dynamics} when a mediator appends correlated codes $c$ to the inputs of both players according to the following reward function:
\begin{equation}
\begin{split}
\label{reward:eq}
r_m= - \sum_{i=0}^T \sum_{j=0}^{T} ReLU(g_\pi(\phi_i,\phi_j, c) + g_{D}(\omega_i,\omega_j,c))
\end{split}
\end{equation}
where  $h^T=(\mu_0...,\mu_T)$ is history of strategies of the game up to time $T$ and $\mu_t=(\phi_t,\omega_t)$, $ReLU(x)= \max(x,0)$ and $g$ is the marginal payoff gain defined in \ref{game_definitions}.
\end{theorem}
\begin{proof}
\small
Dual problem of MeMCorEq is $\inf_{\lambda>=0} \ln(Z(\lambda)$ with the following relationship between the dual variables $\lambda$ and the primal variables $\mu$:
\begin{equation}
\small
\begin{split}
\label{prime_dual_rel:eq}
\ln(Z(\lambda)) = - \ln \bb{E}_{\omega \sim \mu_{D}} [\exp(\sum_{\phi} \sum_{\hat{\phi}} \lambda_{\phi,\hat{\phi}} L(\phi,\omega) -  L(\hat{\phi},\omega) ) ] \\
- \bb{E}_{\phi \sim \mu_{\pi}} [\exp(\sum_{\omega} \sum_{\hat{\omega}} \lambda_{\omega,\hat{\omega}} L(\phi,\omega) -  L(\phi,\hat{\omega)} ) ]
\end{split}
\end{equation}
Instead of learning the Lagrangian multipliers $\lambda$, the mediator learns to introduce code $c$ to the loss function $L(\phi,\omega, c)$ as if $\lambda$ is absorbed into the loss function. To make sure that the Lagrangian constraints $\lambda>=0$ are satisfied, we introduce ReLU function to the equation. This followed by the Jenson inequality yields:
\begin{equation}
\small
\begin{split}
\label{prime_dual_rel:eq}
\ln(Z(c)) \leq \\  \bb{E}_{\omega \sim \mu_{D}} [\sum_{\phi} \sum_{\hat{\phi}} ReLU( L(\hat{\phi},\omega,c) - L(\phi,\omega,c) ) ]   
+ \bb{E}_{\phi \sim \mu_{\pi}} [\sum_{\omega} \sum_{\hat{\omega}}  ReLU ( L(\phi,\hat{\omega},c) - L(\phi,\omega,c) ) ] = \\
\sum_{i=0}^T \sum_{j=0}^{T} ReLU(g_\pi(\phi_i,\phi_j, c)) + ReLU(g_{D}(\omega_i,\omega_j,c))
\end{split}
\end{equation}

The last equality comes from the fact that MeMCorq is a stricter notion than CCE and therefore like other no-regret learning algorithms, the average of past strategies can be used as a proxy for computing MNE $\mu^*$. The proof is then complete using the result of Lemma \ref{convergence:lemma} and definition of MeMCorq.

\end{proof}

In other words, instead of introducing entropy regularizer $h(.)$ to FTRL, $\mathcal{M}$ learns to minimize $D_{KL}(\mu^* || \mu_t)$ by incentivizing players along the way through augmentation of their observations via generate codes $c_t$.
The loss of the modified game is then $L(\phi_t,\omega_t, c_t)$.
We refer to such a no-regret algorithm follow the neurally perturbed leader (FTNPL) since it can be viewed as FTPL with a neural network agent $\mathcal{M}$ as an oracle. The scheme of FTNPL is visualized in Fig. \ref{fig:scheme}.

\begin{algorithm}[t] 
\small
   \caption{Follow the Neurally-Perturbed Leader (FTNPL)}
   \label{alg:FTNPL}
\begin{algorithmic}
   \STATE {\bfseries Input:}.
   Code size $C$, queue size $K$.
 \STATE {\bfseries Initialize:}  Initial parameters $(\phi_0, \omega_0, \psi_0)$ for agent $\pi$, discriminator $D$ and mediator $\mathcal{M}$ respectively. 
Initial observable information in the game $\mathcal{I}_0$, empty queues of size $K$ $h_D= h_{\pi}=\emptyset$, $h_{\pi}$.insert($\phi_0$), $h_{D}$.insert($\omega_0$),  $\bold{u}_D=[], \bold{u}_{\pi}=[]$
 
   \FOR{$i=0, 1, 2, ... $}
        \STATE  $c_i = \mathcal{M}_{\psi_i}(\mathcal{I}_i)$

       \begin{multicols}{2}
           \FOR{$\phi \in h_{\pi}$}
              \STATE  $$\bold{u}_{\pi}.\text{append}(L(\phi_i,\omega_i, c_i))$$
            \ENDFOR  
      $$\phi_{i+1} \leftarrow \nabla_{\phi_i} \sum \bold{u}_{\pi}  $$
      \FOR{$\omega \in h_{D}$}
         \STATE  $$\bold{u}_D.\text{append}(-L(\phi_i,\omega_i, c_i))$$
       \ENDFOR  
     $$\omega_{i+1} \leftarrow \nabla_{\omega_i} \sum \bold{u}_D  $$
       \end{multicols}
        \STATE $\psi_{i+1} \leftarrow -\nabla_{\psi_i} \hat{\mathbb{E}}_{\chi_i}  \log \mathcal{M}_{\psi_i}(\mathcal{I}_i )r_m(\bold{u}_D, \bold{u}_{\pi})$ with $r_m$ defined in eq. \ref{reward:eq}
       \STATE  $h_{\pi}$.insert($\phi_{i+1}$), $h_D$.insert($\omega_{i+1}$), $\bold{u}_{\pi}=[]$, $\bold{u}_D=[]$,

   \ENDFOR
\end{algorithmic}
\end{algorithm}
\subsection{FTNPL Implementation}
The description of the algorithm is given in \ref{alg:FTNPL}. At every time step, the mediator uses the available information in the game $\mathcal{I}_t$ to generate correlated codes $c_t$. 
$\mathcal{I}$ can take the form of pair of latest strategies of the games $(\phi,\omega)$ (in the case of games), pair of observations and actions $(s, a)$ (in the case of imitation learning) or the real data in the form of generative networks.
Both players update their parameters using an FTL algorithm. At every step of the game, the mediator updates its parameters according to the reward function in Eq. \ref{reward:eq}. In practice, we use the second power of $g$ instead of $ReLU$ function. We also parameterize the mediator policy using the reparameterization trick \cite{Blum2015VariationalDA}. In the convex case (e.g Ex. \ref{example1}), mediator action is implemented as the mean of the parameterized policy distribution. However, in the non-convex case (e.g Ex. \ref{example2}), mediator actions have to be random samples from the parameterized policy distribution for the game to converge. For all the experiments except Fig. \ref{fig:FTNPL}, we implemented the mediator action as random samples rather than mean. 

In practice, we keep the queue size $K$ small since runtime and memory of FTNPL algorithm grow linearly with $K$. Unlike \cite{Grnarova2017AnOL}, FTNPL requires no special queuing update. 
Thanks to the theoretical guarantees of FTNPL, none of the previous GAN training hacks such as choices of entropy regularization, grad penalty, or parameter clipping for the discriminator is required. FTNPL removes the recurrent dynamics as well as difficulties of past training methodologies. 

\section{Applications}
We chose $K=5$ and a code-size of $C=2$ for all the experiments. Note that correlated codes intuitively represent the Lagrangians of the optimization problem for both players and therefore $C=2$ is a suitable choice.

\subsection{Matching Pennies Game }
We applied FTNPL to example \ref{example1} and \ref{example2} with $\mathcal{I}$ being the pair of latest strategies of the games $(\phi,\omega)$. It converges to MNE $\mu^*$ in both cases. The training dynamics do not exhibit recurrent dynamics as shown in Fig. \ref{fig:FTNPL} and converges to MNE even under non-convex losses as shown in Fig. \ref{fig:FTNPL_nonconvex_1code} and Fig. \ref{fig:FTNPL_nonconvex} with code size of $C=1$ and $C=2$ respectively. 

\subsection{Generative Adversarial Networks}
\begin{figure}
\centering
\begin{subfigure}[h]{0.23\textwidth}
\includegraphics[width=\textwidth]{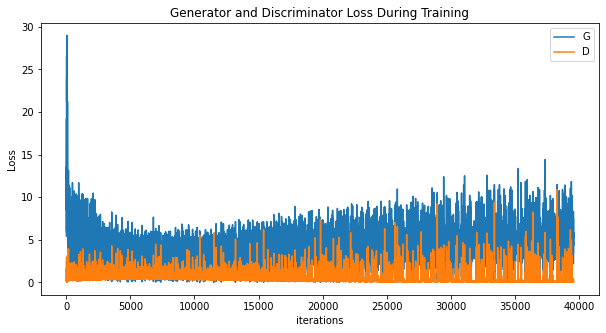}
\caption{\tiny GAN training dynamic applied to celebrity data. }
\label{fig:gan_dynamic}
\end{subfigure}
~
\begin{subfigure}[h]{0.22\textwidth}
\includegraphics[width=\textwidth]{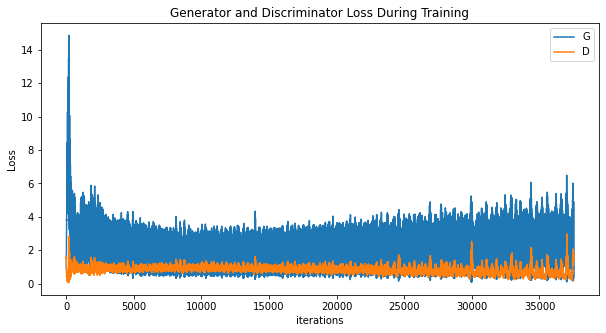}
\caption{\tiny FTNPL-GAN training dynamic applied to celebrity data. }
\label{fig:ftnpl_gan_dynamic}
\end{subfigure}
~
\begin{subfigure}[h]{0.23\textwidth}
\includegraphics[width=\textwidth]{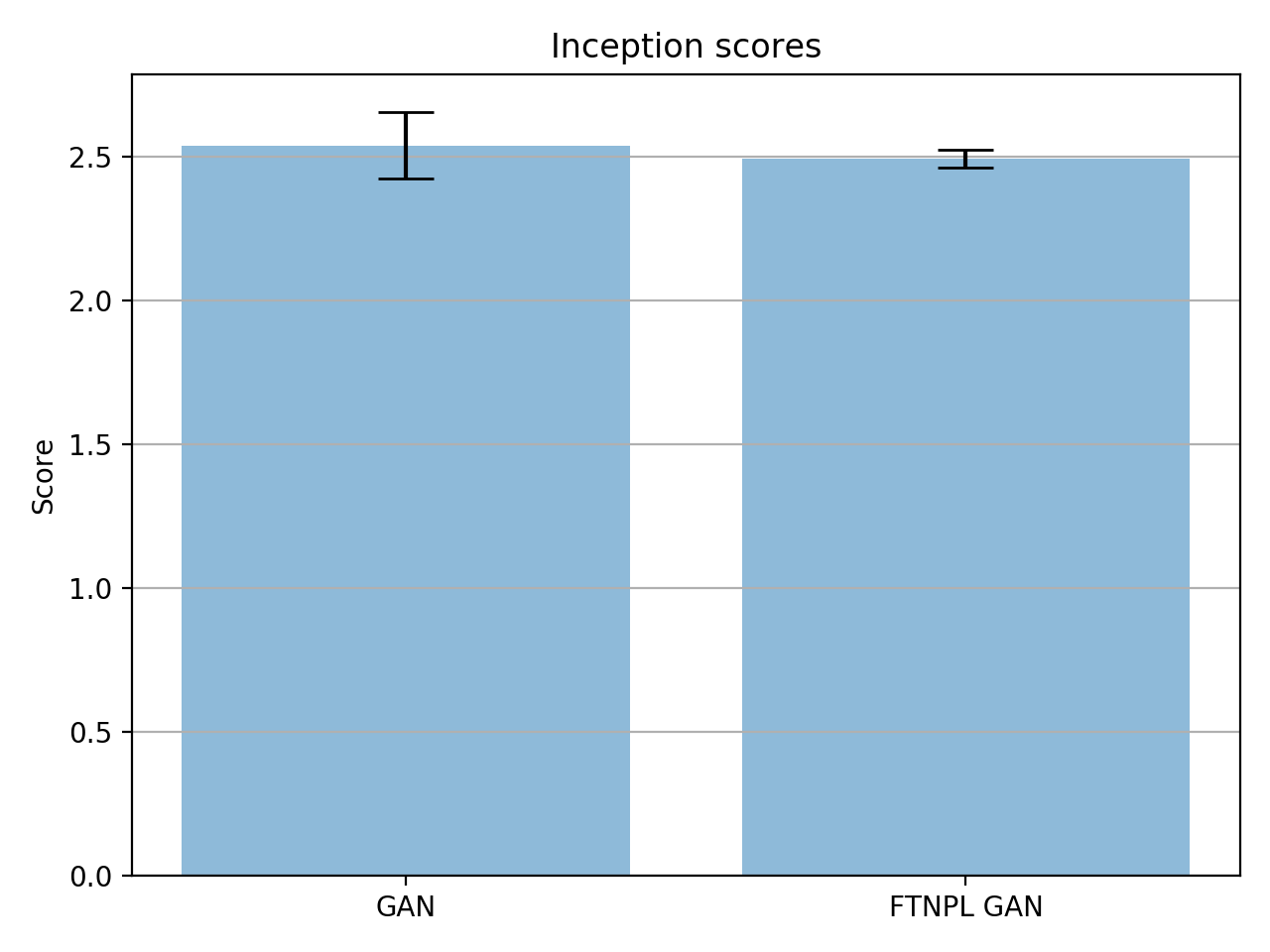}
\caption{\tiny Inception scores after 25 epochs.}
\label{fig:inception}
\end{subfigure}
~
\begin{subfigure}[h]{0.23\textwidth}
\includegraphics[width=\textwidth]{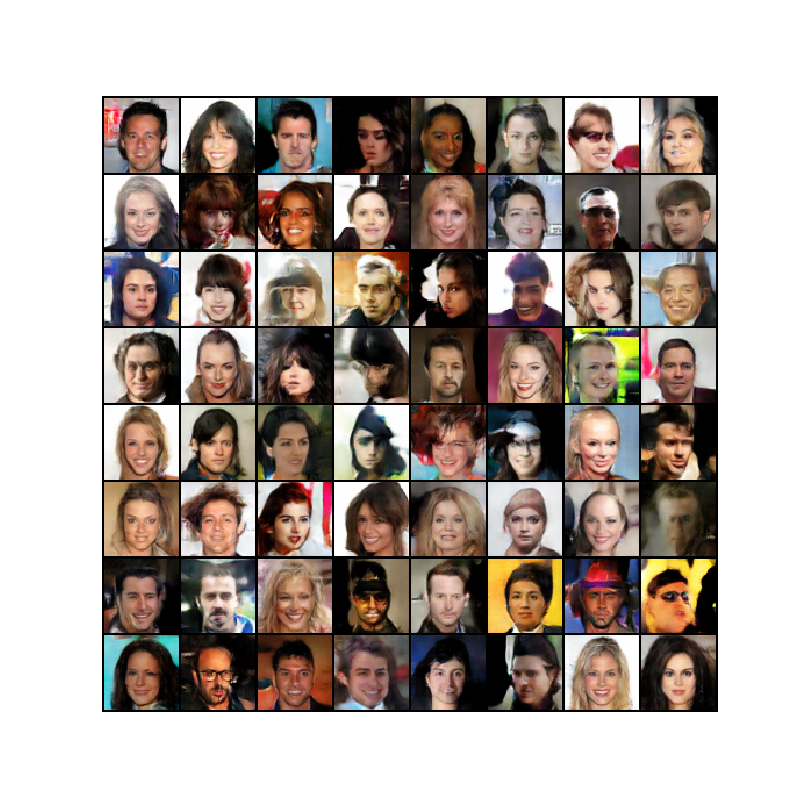}
\caption{\tiny Generated images after 25 epochs.}
\label{fig:generated_images}
\end{subfigure}
\caption{\small FTNPL applied to GAN.}
\label{fig:classic}
\end{figure}

Generative adversarial networks (GAN) \cite{Goodfellow2014GenerativeAN} is a well-known zero-sum deep learning architecture capable of generating synthetic samples from arbitrary distribution $p_\text{data}$. Player $\pi$ gets random noise $z$ as input and generates synthetic data $\pi_{\phi}(z)$. Player $D$ then assigns scores $D_{\omega}(\pi_{\phi}(z))$. $L(\phi,\omega)$ is the Jenson-Shannon distance between $\pi_{\phi}(z)$ and $p_\text{data}$.
We applied FTNPL to DCGAN \cite{Radford2015UnsupervisedRL} with $\mathcal{I}$ being the real image data. Our experiment using celeb data \cite{liu2015faceattributes} shows that training dynamic of FTNPL is smoother compared with that of DCGAN as demonstrated in Fig. \ref{fig:ftnpl_gan_dynamic} and \ref{fig:gan_dynamic}. We stoped training for both models after 25 epochs and computed their inception scores \cite{Salimans2016ImprovedTF}. Fig. \ref{fig:inception} shows that FTNPL leads to smaller variance for the inception score. Some generated samples are visualized in Fig. \ref{fig:generated_images}.
\subsection{Adversarial Imitation Learning}
\begin{algorithm}[h] 
\small
   \caption{Correlated Rollout}
   \label{alg:rollout}
\begin{algorithmic}
  \begin{multicols}{2}
   \STATE {\bfseries Input:} policy network $\pi$, mediator policy network $\mathcal{M}$, $i =0$, number of steps $n$, an initial code $c_0 = 0_C$ with $C$ the size of codes,
   initial state-action trajectory $\tau=[]$,  initial code trajectory $\tau_c=[]$,  environment env, $s_0$ = env.reset()
   \STATE \textbf{Output: } state-action trajectory $\tau$ , code trajectory $\tau_c$.
   \WHILE{ not done and $i<n$}
   \STATE $a = \pi(s_i,c_i)$
   \STATE $\tau_c$.append($c_i$), $\tau$.append($s_i$, $a_i$)
    \STATE $i+=1$
   \STATE $s_i$, done  = environment($a$)
    \STATE $ c_i = \mathcal{M}(s_i,a)$
  
   \ENDWHILE
     \end{multicols}
\end{algorithmic}
\end{algorithm}

Training dynamics of generative adversarial imitation learning (GAIL) \cite{Ho2016GenerativeAI} is more complicated than GAN. This is because the agent environment is a black box and this makes the optimization objective to be non-differentiable end-to-end. As a result, proper policy \textit{rollouts} and Monte-Carlo estimation of policy gradients are required which makes the training dynamic more complicated. Therefore, the rest of the experiments are focused on the application of the FTNPL to GAIL. 
To formally explain the GAIL let $(\mathcal{S}, \mathcal{A}, P, \mathit{r}, \rho_0, \gamma)$ be an infinite-horizon, discounted Markov decision process (MDP) with state-space $\mathcal{S}$ , action space $\mathcal{A}$, transition probability distribution $P:\mathcal{S}\times\mathcal{A}\times\mathcal{S}\to\mathbb{R}$, reward function $\mathit{r}:\mathcal{S}\to\mathbb{R}$, distribution of the initial state $s_0$ $\rho_0:\mathcal{S}\to\mathbb{R}$, the discount factor $\gamma\in(0,1)$.
In the case of imitation learning, we are given access to a set of expert trajectories $\tau_E$  that are achieved using expert policy $\pi_E$. We are interested at estimating a stochastic policy $\pi:\mathcal{S}\times\mathcal{A}\to[0,1]$. To estimate $\pi_E$, GAIL optimizes the following: 
\begin{equation}
\begin{split}
\label{gail:eq}
\min_{\pi}\max_{D\in (0,1)^{\mathcal{S}\times\mathcal{A}}}\bb{E}_{\pi}[\log D(s, a)] + 
\bb{E}_{\pi_E}[\log (1 - D(s, a))] - \lambda H(\pi)
\end{split}
\end{equation}
, with the expected terms defined as $\bb{E}_\pi[D(s, a)]\triangleq\bb{E}[\sum_{t=0}^\infty\gamma^tD(s_t, a_t)]$, where $s_0 \sim \rho_0$, $a_{t} \sim \pi(a_t | s_t)$, $s_{t+1} \sim P(s_{t+1} | a_t, s_t)$,  and $H(\pi) \triangleq \bb{E}_{\pi}[-\log \pi(a | s)]$ is the $\gamma$-discounted causal entropy. An adversary player $D$ tries to distinguish state-action pairs generated during \textit{rollout} using $\pi$ from the demonstrated trajectories generated by $\pi_E$. To apply FTNPL Alg. \ref{alg:FTNPL} to GAIL, we modify the rollout algorithm to Alg. \ref{alg:rollout} with $\mathcal{I}= (s,a)$. The concept of adding codes to the policy network is similar to infoGAIL \cite{Li2017InfoGAILII}. InfoGAIL uses fixed code to guide an entire trajectory. Moreover, it uses other regularization terms in the policy gradient optimization objective, to make sure that the codes would not be ignored. The correlated codes have different properties. First, they are generated per state-action (they also are fed into the discriminator) and therefore it addresses the multimodality and other types of variations within the trajectories as well. Second, there is no need to include any extra regularization terms including discounted causal entropy, i.e we assume $\lambda =0$ in Eq. \ref{gail:eq}. Agent $\pi$ uses PPO \cite{Schulman2017ProximalPO} for updates. We also used the utility definition of Wasserstein GAN \cite{arjovsky2017wasserstein} for our final implementation. 
 The rest is a straightforward application of FTNPL \ref{alg:FTNPL} to GAIL. We also apply FTPL algorithm to GAIL as a baseline.
\subsubsection{Expriements}
\begin{figure}
\centering
\begin{subfigure}{0.23\textwidth}
\includegraphics[width=\textwidth]{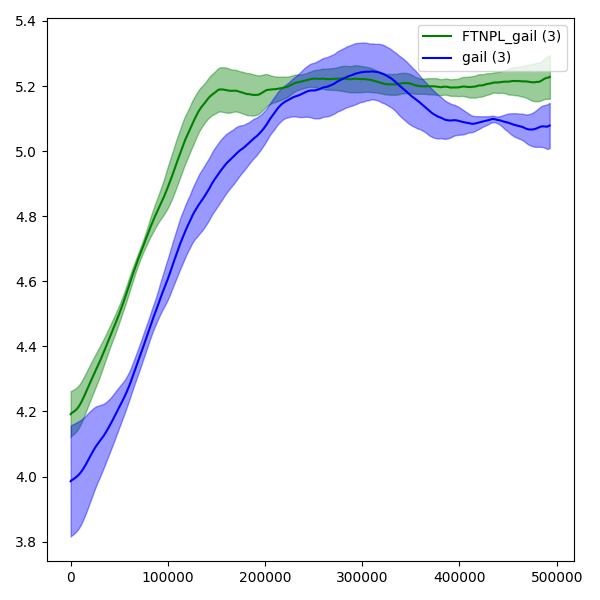}
\caption{\tiny Coinrun with One level}
\label{fig:one_level}
\end{subfigure}
~
\begin{subfigure}{0.23\textwidth}
\includegraphics[width=\textwidth]{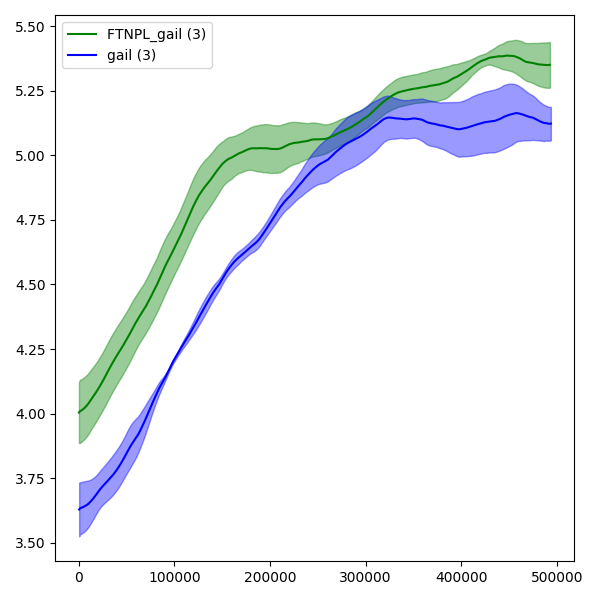}
\caption{\tiny Coinrun with Unbounded levels}
\label{fig:all_levels}
\end{subfigure}
~
\begin{subfigure}{0.23\textwidth}
\includegraphics[width=\textwidth]{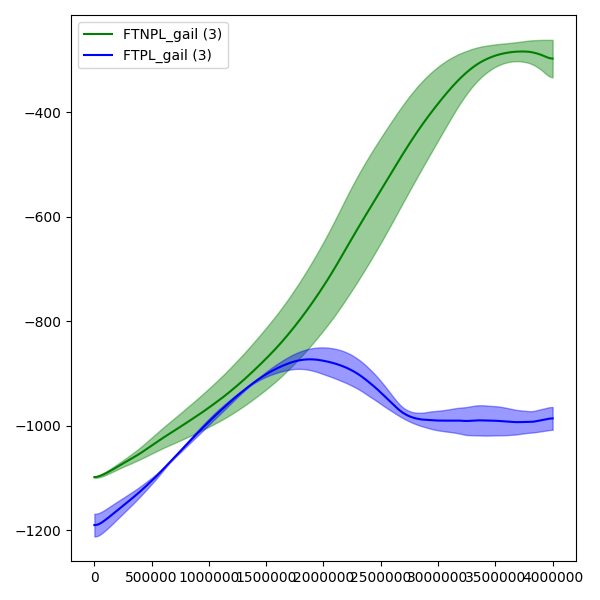}
\caption{\tiny Pendulum}
\label{fig:pend}
\end{subfigure}
~
\begin{subfigure}{0.23\textwidth}
\includegraphics[width=\textwidth]{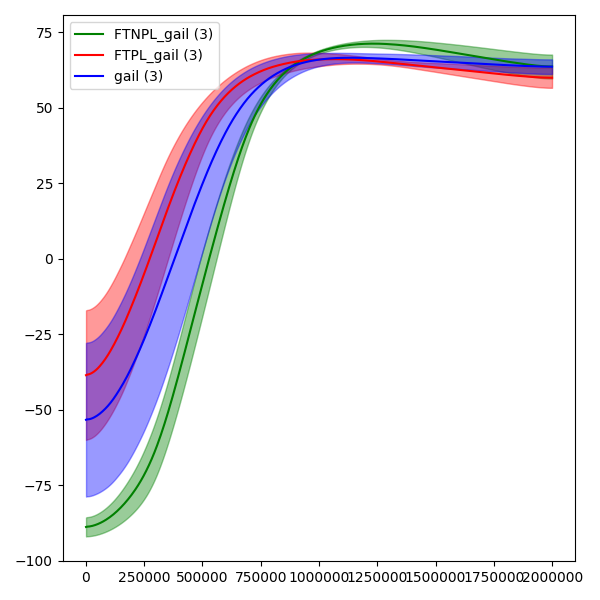}
\caption{\tiny Continuous Mountain Car}
\label{fig:car}
\end{subfigure}
\caption{\small FTNPL applied to Imitation Learning.}
\end{figure}

\textbf{Changing environments} When learning from the expert trajectories $\pi_E$, one has to take into account the \textit{internal} factors of variations such as learning in new environments that not are demonstrated in the expert trajectories. For example, observations can vary for different levels of a game. Using the CoinRun environment \cite{Cobbe2018QuantifyingGI}, we set up an experiment to show that the introduction of correlated codes in FTNPL not only can address recurrent dynamic issues but also the internal factors of variations. CoinRun is a procedurally generated environment that has a configurable number of levels and difficulty. It can provide insight into an agent's ability to generalize to new and unseen environments. The game observations are high dimensional ( $64 \times 64 \times 3 $ RGB) and therefore it is also suitable for testing the efficiency of RL-based mediator of FTNPL under a complex environment. 
The only reward in the CoinRun environment is obtained by collecting coins, and this reward is a fixed positive constant.  A collision with an obstacle results in the agent's death and levels vary widely in difficulty. The level terminates when the agent dies, the coin is collected, or after 1000 time steps. 
We first trained a PPO \cite{Schulman2017ProximalPO} agent using the 3-layer convolutional architecture proposed in \cite{Mnih2015HumanlevelCT}. We stopped the training when it achieved a reward of 6.3 with 500 levels. We then generated expert trajectories with the same number of levels. However, in imitation learning training, two different number of levels are selected: one and unbounded set of levels. A higher number of levels decrease the chance that a given environment gets encountered more than once. For the unbounded number of levels, this probability is almost 0.
The selection of these different numbers of levels provides an insight into the adaptability and transferability of the imitation algorithms to new environments.
 We visualized the performance of FTNPL GAIL and GAIL, averaged over three different seeds.  When there is only one level, the gap between the sample efficiency of GAIL and FTNPL-GAIL is not very wide as shown in Fig. \ref{fig:one_level}. However, with the increase in the number of levels (i.e changing environments), the FTNPL training outperforms GAIL more noticeably as demonstrated in \ref{fig:all_levels}. 
We also did a few classic baselines for MountainCarContinuous-v0 in Fig. \ref{fig:car} and pendulum-v0 in Fig. \ref{fig:pend}. FTNPL outperforms the other baselines and also has smaller variances across all the experiments. 

\textbf{Imitating mixture of state-action trajectories}
Another type of variation in $\pi_E$ is the \textit{external} one, such as when one is learning from a mixture of expert demonstrations. For this, we used the Synthetic 2D-CircleWorld experiment of \cite{Li2017InfoGAILII}. 
The goal is to select direction strategy at time $t$ using the observations of $t-4$ to $t$ such that a path would mimic those demonstrated in expert trajectories $\tau_E$. These expert trajectories are stochastic policies that produce circle-like trajectories. They contain three different modes as shown in Fig. \ref{fig:circleworld}. A proper imitation learning should have the ability to distinguish the mixture of experts from each other. The results in Fig. \ref{fig:circleworld} demonstrate the path of learned trajectories during the last 40K of the overall 200K steps of training. It can be seen that FTNPL-GAIL can distinguish the expert trajectories and imitate the demonstrations more efficiently than FTPL-GAIL and GAIL. 

\begin{figure}
\centering
\begin{subfigure}[b]{0.23\textwidth}
\includegraphics[width=\textwidth]{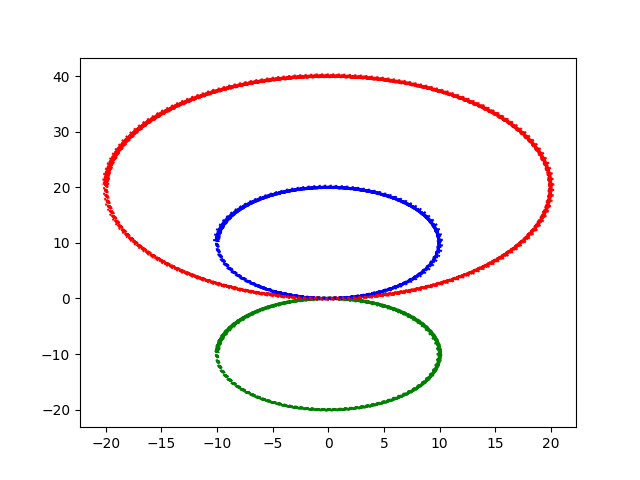}
\caption{\small Expert}
\end{subfigure}
~
\begin{subfigure}[b]{0.23\textwidth}
\includegraphics[width=\textwidth]{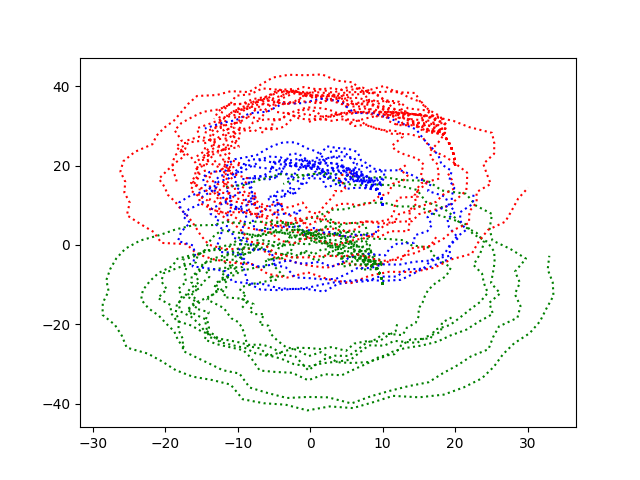}
\caption{\small GAIL}
\end{subfigure}
~
\begin{subfigure}[b]{0.23\textwidth}
\includegraphics[width=\textwidth]{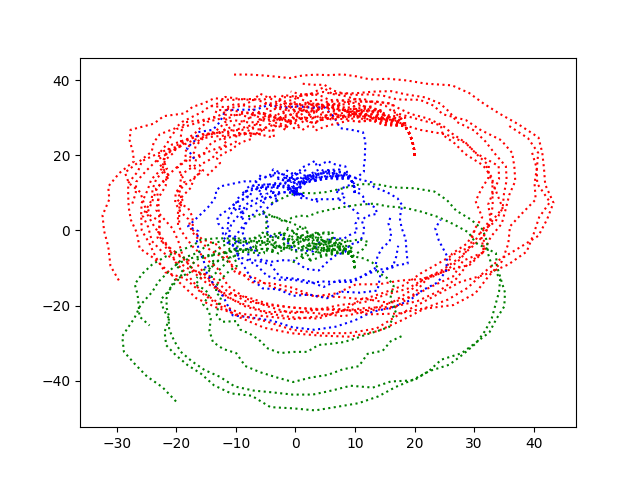}
\caption{\small FTPL GAIL}
\end{subfigure}
~
\begin{subfigure}[b]{0.23\textwidth}
\includegraphics[width=\textwidth]{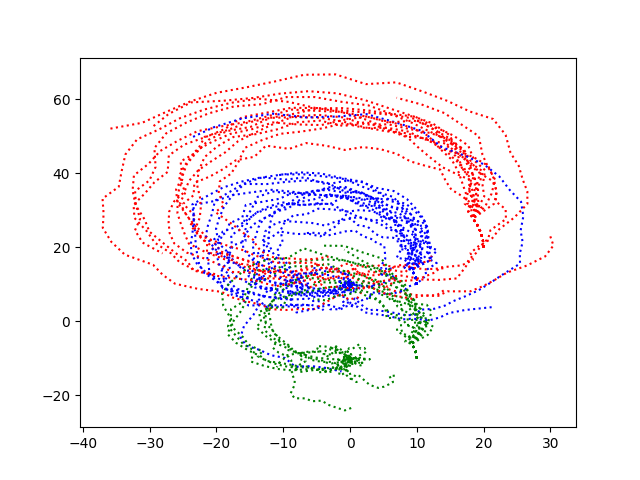}
\caption{\small FTNPL GAIL}
\end{subfigure}
\caption{\small Path of trajectories during training: The results are for the last 40K of the overall 200K steps.}
\label{fig:circleworld}
\end{figure}
\nocite{pytorchrl}

\section{Conclusion} 
We proposed a novel follow the perturbed leader algorithm for training adversarial architectures. The proposed algorithm guarantees convergence to mixed Nash equilibrium without recurrent dynamics and the loss-convexity assumption. As demonstrated in our experiments, it is also suitable for training in environments with various factors of variations thanks to the introduction of correlated codes. 

\bibliographystyle{neurips_2020}
\bibliography{neurips_2020}

\end{document}